\documentclass{article}
\usepackage{ijcai19}

\usepackage{times}
\usepackage{soul}
\usepackage{url}
\usepackage{amsfonts}
\usepackage[hidelinks]{hyperref}
\usepackage[utf8]{inputenc}
\usepackage[small]{caption}
\usepackage{graphicx}
\usepackage{subfigure}
\usepackage{amsmath}
\usepackage{booktabs}
\usepackage{algorithm}
\usepackage{algorithmic}
\usepackage{amsmath}
\newtheorem{theorem}{Theorem}
\usepackage{booktabs}
\usepackage{makecell}
\newenvironment{proof}{{\noindent \textbf{Proof}}\quad}{\hfill $\square$\par}
\usepackage{multirow}
\title{DANE: Domain Adaptive Network Embedding}
\author{
Yizhou Zhang$^1$\and
Guojie Song$^{1,2}$\footnote{Corresponding author}\and
Lun Du$^{1,2}$\and
Shuwen Yang$^1$\And
Yilun Jin$^1$\\
\affiliations
$^1$School of Electronic Engineering and Computer Science, Peking University
\\
$^2$Key Laboratory of Machine Perception (Ministry of Education), Peking University
\emails
\{zhangyizhou2015, dulun, gjsong, swyang, yljin\}@pku.edu.cn
}

\begin{document}

\maketitle

\begin{abstract}

Recent works reveal that network embedding techniques enable many machine learning models to handle diverse downstream tasks on graph structured data. However, as previous methods usually focus on learning embeddings for a single network, they can not learn representations transferable on multiple networks. Hence, it is important to design a network embedding algorithm that supports downstream model transferring on different networks, known as domain adaptation. In this paper, we propose a novel Domain Adaptive Network Embedding framework, which applies graph convolutional network to learn transferable embeddings. In DANE, nodes from multiple networks are encoded to vectors via a shared set of learnable parameters so that the vectors share an aligned embedding space. The distribution of embeddings on different networks are further aligned by adversarial learning regularization. In addition, DANE's advantage in learning transferable network embedding can be guaranteed theoretically. Extensive experiments reflect that the proposed framework outperforms other well-recognized network embedding baselines in cross-network domain adaptation tasks.

\end{abstract}

\section{Introduction}

Network embedding, which learns low-dimensional embedding vectors for nodes from networks, is an important technique enabling the applications of machine learning models on network-structured data \cite{Perozzi2014DeepWalk,Tang2015LINE,Grover2016node2vec,Hamilton2017Inductive}. It learns to preserve structural and property similarity in embedding space and can support different downstream machine learning tasks, such as node classification \cite{Yang2016Semi} and network visualization \cite{Tang2016Visual}. 

However, most of existing methods mainly focus on learning representations for nodes from a single network. 
As a result, when handling multiple networks, they suffer from embedding space drift \cite{Du2018DNE} and embedding distribution discrepancy \cite{Tzeng2017Adv}, resulting in decreased accuracy when transferring the downstream machine learning models across networks to handle the same task.
As the branch of transfer learning that handles the same task on different datasets is usually known as domain adaptation \cite{Mansour2009Domain}, we define the network embedding algorithms that can support such transfer learning on different networks as domain adaptive network embedding, which brings following advantages.
The first advantage is that it alleviates the cost of training downstream machine learning models by enabling models to be reused on other networks. The second advantage is that it handles the scarcity of labeled data by transferring downstream models trained well on a labeled network to unlabeled networks. Apart from the above advantages, compared with traditional discriminative domain adaptation methods requiring label information from source domain, domain adaptive network embedding learns representations in an unsupervised manner, requiring no label information from neither the source network nor the target network. Therefore, it makes no difference which network is the source network in downstream tasks, enabling bidirectional domain adaptation.

Currently, most researches of domain adaptation concentrate on CV and NLP fields \cite{Long15DAN,Fu2017Domain}. 
existing domain adaptation methods can not be directly applied on network embedding problems. First, 
these methods are usually designed for CV and NLP tasks, where
samples, e.g. images and sequences are independent and identically distributed, resulting in little requirement for model rotational invariance \cite{Khasanova17Invariant}. However, network structured data, where nodes are connected with edges representing their relations, require models with rotational invariance because of the phenomenon known as graph isomorphism \cite{DBLP:conf/nips/DefferrardBV16}. 
Therefore, existing methods can not model network structural information, which is the core of network embeddings. Second, most existing domain adaptation models learn discriminative representations in a supervised manner \cite{Tzeng2017Adv}, where the value of loss function is only associated with each single sample's absolute position in their feature space. Network embedding, alternatively, usually aims to learn multipurpose representations in an unsupervised manner by preserving the relative position of all node pairs, resulting in increased difficulty in optimization. As a result, more stable model architecture and loss functions are required to alleviate such difficulty.

Essentially, domain adaptive network embedding suffers from two major challenges, the more straightforward of which is the \emph{embedding space alignment}, which means that structurally similar nodes should have similar representations in the embedding space, even if they are from different networks. However, many typical network embedding methods only preserve the structural similarity within a single network \cite{Heimann2017MLG}
which are not applicable for cross-network node pairs. The other less explicit challenge is that \emph{distribution shift} of embedding vectors also influences the performance of model on target networks, because most machine learning models perform as guaranteed only when they work on data with similar distribution as training data.

In this paper, we propose DANE, an unsupervised network embedding framework handling embedding space drift and distribution shift in domain adaptation via graph convolutional network \cite{KipfW16} and adversarial learning \cite{Goodfellow14GAN}. 
To enable GCN to preserve the structural similarity of cross-network node pairs, we apply shared weight architecture, which means that the GCN embeds nodes from both source network and target network via shared learnable parameters.

The distribution shift in domain adaptation is handled via an adversarial learning component based on least square generative adversarial network \cite{Mao2017LSGAN} . Compared with the original GAN used in existing domain adaptation models, the loss function and architecture of LSGAN can generate gradients with larger scale for samples lying a long way to the decision boundary, which relieves the problem of vanishing gradients when models are close to convergence. 

In summary, our contributions are:
\begin{itemize}
\item We formulate the task of designing a domain adaptive network embedding framework. It is beneficial for transferring models across multiple networks, which has not been explored by previous network embedding methods.
\item We propose DANE, a framework that can achieve embedding space alignment and distribution alignment via shared weight graph convolutional network and adversarial learning regularization based on LSGAN.
\item We constructed two datasets to test network embedding methods' performance on the task of supporting domain adaptation and conducted experiments on DANE and other baselines via these datasets. The result indicate that our model has leading performance in this task.
\end{itemize}

\section{Related Work}

\textbf{Network Embedding.} Network embedding maps the vertices or edges of a network into a low-dimensional vector space. Such mapping can benefit the application of many machine learning models in many downstream tasks on network structured data\cite{Yang2016Semi,Tang2016Visual,8594902,MEgo,WangLSFDL19}. Existing methods include transductive methods and inductive methods. Transductive methods directly optimize the representation vectors. They usually apply matrix factorization \cite{Li2019Sep} 
or Skip-Gram model inspired by word2vec \cite{Perozzi2014DeepWalk,Tang2015LINE,Grover2016node2vec}. Inductive methods learn functions which take the structural information and node features as input and output their representation vectors. They usually model the mapping function via deep neural networks like graph convolutional networks\cite{Kipf2016VGAE,Hamilton2017Inductive}. All the methods above mainly focus on representing a single network better, without considering the domain adaptation on multiple unconnected networks.
	
\textbf{Domain Adaptation.} Domain adaptation aims to learn machine learning models transferable on different but relevant domains sharing same label space \cite{Mansour2009Domain}. Recent research mainly focus on learning domain invariant representation via neural networks so that deep learning models trained on labeled source datasets can be transferred to a target datasets with few or no labeled samples \cite{Ganin2015Unsupervised,Long15DAN,Peng2016Multi,Fu2017Domain,Tzeng2017Adv}. 
Above methods are mostly designed for image or text while not considering the application on graph structured data because of their lack of rotational invariance and limits in handling graph isomorphism. Only a few methods \cite{DAGraph,GraphMatchDA} make use of information of a similarity graph whose edges represent artificially designed similarity (like dot product of representation vectors) of independent samples (text or images), instead of real links between relevant nodes.

\section{Problem Definition}

\begin{table}[bp]
\centering
\begin{tabular}{c||c}
\Xhline{1.5pt}
Notations & {\bf Definition}\\
\Xhline{1.5pt}
$G$ & Network\\
\hline
$G_{src}$ &  Source network \\
\hline
$G_{tgt}$ & Target network \\
\hline
$N_{G}$ &  Node set of network $G$ \\
\hline
$E_{G}$ & Edge set of network $G$ \\
\hline
$X_{G}$ & Node feature matrix of network $G$ \\
\hline
$Z$ & Embedding space  \\
\hline
$V_{G}$ & Embedding vector set of network $G$ \\
\hline
$p_{G}(z)$ & Probability density function of $V_{G}$ in $Z$  \\
\Xhline{1.5pt}
\end{tabular}
\caption{\label{tab:unsurpervised} Notations.}
\end{table}

\textbf{Definition 1: Domain Adaptation on Networks.} 
Domain adaptation on networks aims to train a machine learning model $M$ for a downstream task by minimizing its loss function on $G_{src}$ and ensure that $M$ can also have good performance when we transfer it on $G_{tgt}$ to handle the same task. 

To this end, following two constraints need to be satisfied:

\begin{figure*}[tbp]
\centering
\includegraphics[width=6.5in]{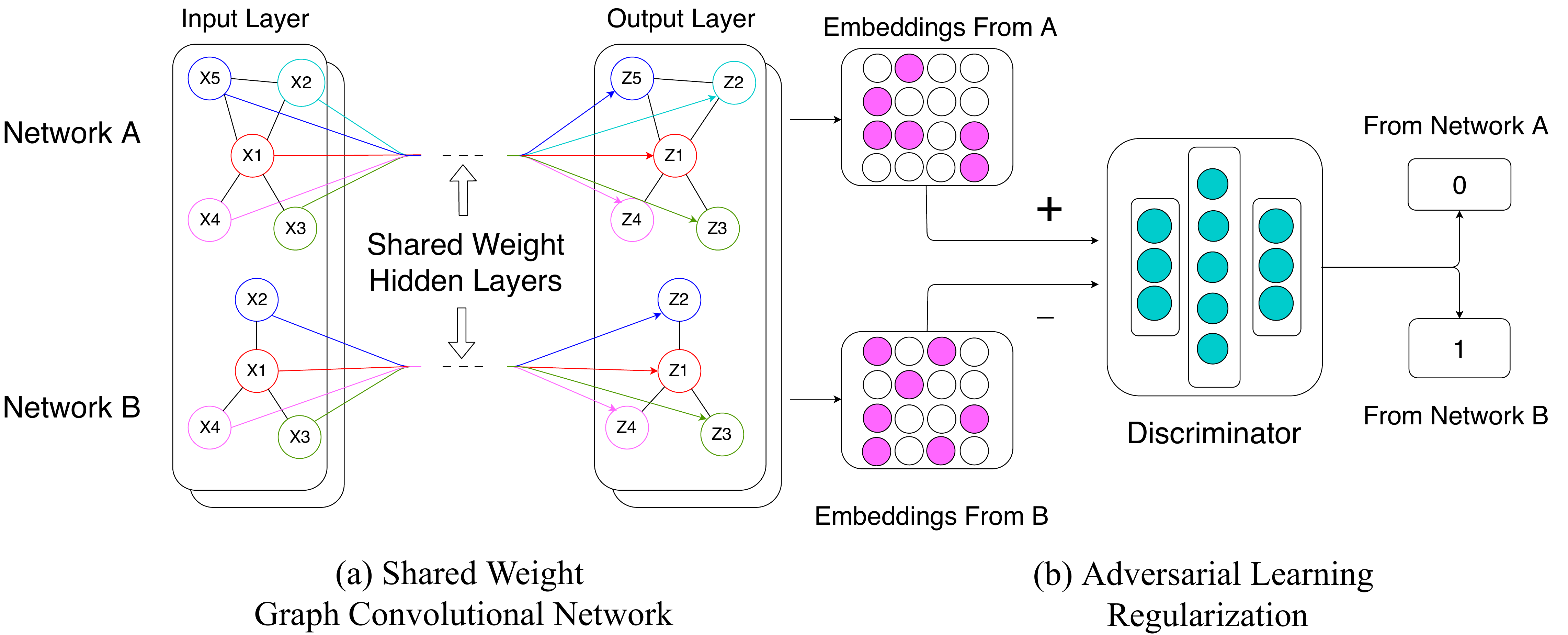}
\caption{An overview of DANE. DANE consists of two major components: (a) shared weight graph convolutional network (SWGCN) projects the nodes from two networks into a shared embedding space and preserve cross-network similarity; (b) adversarial learning regularization is a two-player game where the first player is a discriminator trained to distinguish which network a representation vector is from and the second player is the SWGCN trying to generate embeddings that can confuse the discriminator.}
\label{fig:1}
\end{figure*} 

\begin{itemize}
\item \textbf{Embedding Space Alignment.}
Embedding space alignment aims to project the nodes of $G_{src}$ and $G_{tgt}$ into a shared embedding space $Z$, where structurally similar nodes have similar representation vectors even if they are from different networks, so that M can be transferable on $G_{src}$ and $G_{tgt}$.
\item \textbf{Distribution Alignment.}
Distribution alignment aims to constrain $p_{src}(z)$ and $p_{tgt}(z)$ close almost everywhere in $Z$, so that $V_{src}$ and $V_{tgt}$ can have similar distributions.
\end{itemize}

Obviously, not arbitrary pairs of networks enables bidirectional domain adaptation. In this paper, we mainly deal with networks where edges are homogeneous and node features express the same meaning. We denote such networks as domain compatible networks.

\textbf{Problem Statement: Domain Adaptive Network Embedding.} Given two domain compatible networks $G_{A}$ and $G_{B}$, domain adaptive network embedding aims to learn network embedding $f_G: N_{A} \bigcup N_{B} \rightarrow \mathbb{R}^d $ 
in an unsupervised manner, which can support bidirectional domain adaptations from $G_A$ to $G_B$ and reciprocally as well, by achieving embedding space alignment and distribution alignment.

\section{Proposed Method}

\subsection{Overall Framework}

To enable domain adaptation on networks, we propose domain adaptive network embedding (DANE), a model that leverages a Shared Weight Graph Convolutional Network to achieve embedding space alignment and apply Adversarial Learning Regularization to achieve distribution alignment. 
The overview of our model is shown in Fig 1.
\begin{itemize}
\item \textbf{Shared Weight Graph Convolutional Network.} In order to learn embeddings that preserve the cross-network structural similarity, we use shared learnable parameters when encoding the nodes from two networks to vectors via GCN.
\item \textbf{Adversarial Learning Regularization.} To align the the distribution of embedding vectors from different networks, we apply adversarial learning to force our model to learn embeddings that can confuse the discriminator trained to distinguish which network a representation vector is from.
\end{itemize}

Without loss of generality, we assume $G_A$ is the source network and $G_B$ is the target network. As DANE's architecture and loss function are symmetric, it is able to handle bidirectional domain adaptations.

\subsection{Shared Weight Graph Convolutional Network}
Graph Convolutional Network (GCN) \cite{KipfW16} represents each vertex in a graph as an embedding vector based on node feature matrix $X$ and adjacency matrix $A$. In GCN, each layer can be expressed as follows:
\begin{equation}
H^{(l+1)}=\sigma(\hat{D}^{-\frac{1}{2}}\hat{A}\hat{D}^{-\frac{1}{2}}H^{(l)}W_{l})
\end{equation}
where $\hat{A} = A + I_N$, $\hat{D}_{ii} = \sum_j\hat{A}_{ij}$, $H^{(l)}$ is the output of the $l$-th layer, $H^{(0)} = X$, $\sigma$ is activation function and $W_{l}$ are learnable parameters of the $l$-th layer.
We use a shared parameter set  $\theta_s = \{W_1,W_2,...\}$ to embed both source and target model. As proven by Jure Leskovec et al \cite{donnat2018}, graph convolution operations can preserve the similarity of a node pair if their local sub-networks (i.e. the subgraph consisting of $k$-hop neighbors of a node and the node itself) are similar. Therefore, GCN with such shared weight architecture can preserve local sub-network similarity. Under \emph{structural equivalence} hypothesis in complex network theory \cite{Henderson12RolX,Grover2016node2vec}, two nodes having similar local network neighborhoods can be considered structurally similar even if they are from two different networks. Hence, shared weight graph convolutional network architecture can project the nodes of $G_{src}$ and $G_{tgt}$ into a shared embedding space $Z$, where structurally similar nodes have similar representation vectors.

To learn compatible parameters for both $G_{src}$ and $G_{tgt}$, we apply a multi-task loss function preserving properties on two networks simultaneously:
\begin{equation}
L_{gcn} = L_{G_{src}}+L_{G_{tgt}}
\end{equation}
where $L_{G_{src}}$ denotes the loss function on the source network and $L_{G_{tgt}}$ denotes the same function calculated on the target network. We apply the first-order loss function proposed in LINE \cite{Tang2015LINE}:

\begin{equation}
L_{G} = -\sum_{(i,j)\in E}\log\sigma(v_j \cdot v_i)-Q\cdot\mathbb{E}_{k \sim P_{neg}(N)}\log\sigma(-v_i \cdot v_k)
\end{equation}
where $E$ is the set of edge, $Q$ is the number of negative samples, $P_{neg}$ is a noise distribution where negative samples are drawn and $\sigma$ is the sigmoid function. In this paper, we set $P_{neg}(N)\propto d_n^{0.75}$, where $d_n$ is degree of node $n$. 
As a flexible framework, DANE is able to preserve other network properties like triadic closure \cite{HuangTWLf14} when we replace $L_{G}$ with corresponding loss functions.

\subsection{Adversarial Learning Regularization}

In order to align the distribution of $V_{src}$ and $V_{tgt}$ in embedding space $Z$, we add Adversarial Learning Regularization into our model. We train a discriminator to distinguish which network an embedding vector is from, and train the Shared Weight Graph Convolutional Network to confuse the discriminator, just like what GANs do. Such training method will force $P(v\in V_{src}|v=z)$ and $P(v\in V_{tgt}|v=z)$ to keep close almost everywhere in the embedding space $Z$, which is equivalent to distribution alignment.

In this work, inspired by  LSGAN \cite{Mao2017LSGAN}, we design the architecture and loss function of discriminator based on Pearson $\chi^2$ divergence to avoid the instability of adversarial learning. Discriminator $D$ is a multi-layer perceptron having no activation function in the final layer. We expect $D$ to output $0$ when the input vector is sampled from $V_{src}$ and  output $1$ otherwise. 
So the discriminator's loss function of is:
\begin{equation}
L_{D}=\mathbb{E}_{x \in V_{src}}[(D(x)-0)^2]+\mathbb{E}_{x \in V_{tgt}}[(D(x)-1)^2]
\end{equation}
where $D(x)$ is the output of the discriminator. In LSGAN, on the contrary, the generator confuses the discriminator unidirectionally by forcing the distribution of fake samples to approximate that of real samples. However, hoping to keep DANE's architecture and loss function symmetric so that it can handle bidirectional domain adaptation, we design following adversarial training loss function:
\begin{equation}
L_{adv}=\mathbb{E}_{x \in V_{src}}[(D(x)-1)^2]+\mathbb{E}_{x \in V_{tgt}}[(D(x)-0)^2]
\end{equation}

We combine the training of Shared Weight Graph Convolutional Network and Adversarial Learning Regularization together by defining the overall loss function of DANE as follows:
\begin{equation}
L=L_{gcn}+\lambda L_{adv}
\label{comb}
\end{equation}
where  $\lambda$ is a hyperparameter to control the weight of regularization. In this paper, we set $\lambda$ as $1$. 
In each iteration, we first train the discriminator for $k$ steps to optimize $L_{D}$, followed by training our embedding model for $1$ step to optimize the graph convolutional network based on Equation \ref{comb}.

\subsection{Theoretical Analysis}

In this section, we show that the better embedding spaces and distributions are aligned, the better the downstream model $M$ performs on target network. We provide a theoretical analysis by exploring the relation between alignment effects and the upper bound of the difference between the loss function value on model $M$ on the source network and the target network when handling a node classification task.

We assume that $M$ outputs the conditional distribution of a node's label $y$ based on its representation vector $v$ and model parameter $\theta$, denoted as $P(y|v;\theta)$. $P(y|v;\theta)$ gives the possibility that a node has a label given the embedding vector of the node. The loss function of the model on $G_{src}$, denoted as $L_{src}$, is defined as follows:
\begin{equation}
\begin{aligned}
L_{src}&=\mathbb{E}(D(\hat{P}_{src}(y|v),P(y|v;\theta)))
\end{aligned}
\end{equation}
where $\hat{P}_{src}(y|v)$ is the groundtruth, and $D(P_1,P_2)$ measures the distance of two distributions. Assuming that $G_{src}$ contain as many nodes as being able to approximate $v$ as a continuous variable, we have the following equation:
\begin{equation}
L_{src}=\int_Z p_{src}(z)\cdot D(\hat{P}_{src}(y|z),P(y|z;\theta))dz
\end{equation}

In a similar way, the performance of the same model on $G_{tgt}$ can be measured by following loss function:
\begin{equation}
\begin{aligned}
L_{tgt}&=\mathbb{E}(D(\hat{P}_{tgt}(y|v),P(y|v;\theta)))\\&=\int_Z p_{tgt}(z)\cdot D(\hat{P}_{tgt}(y|z),P(y|z;\theta))dz
\end{aligned}
\end{equation}

We introduce a theorem:

\begin{theorem}

If following inequalities are satisfied:

\begin{equation}
D(\hat{P}_{src}(y|z),\hat{P}_{tgt}(y|z)) < c, \forall z \in Z
\end{equation}

\begin{equation}
\frac{|p_{src}(z) - p_{tgt}(z)|}{p_{src}(z)} < \epsilon, \forall z \in Z
\end{equation}
where $D(P_1,P_2)$ measures the distance of two distributions and satisfy triangular inequality:
\begin{equation}
D(P_1,P_3)\leq  D(P_1,P_2) + D(P_2,P_3)
\end{equation}

Then we will have following inequality:

\begin{equation}
L_{tgt} - L_{src} \leq \epsilon L_{src} + c + c\cdot \epsilon
\end{equation}

\end{theorem}

\begin{proof}
\begin{equation}
\begin{aligned}
L_{tgt} - L_{src} &= \int_Z p_{tgt}(z)\cdot D(\hat{P}_{tgt}(y|z),P(y|z;\theta))dz\\ 
&-\int_Zp_{src}(z)\cdot D(\hat{P}_{src}(y|z),P(y|z;\theta)))dz\\
L_{tgt} - L_{src}&\leq  \epsilon\int_Z p_{src}(z)\cdot D(\hat{P}_{src}(y|z),P(y|z;\theta))dz\\
&+\int_Zp_{src}(z)\cdot D(\hat{P}_{src}(y|z),\hat{P}_{tgt}(y|z)))dz\\
&+\epsilon\int_Zp_{src}(z)\cdot D(\hat{P}_{src}(y|z),\hat{P}_{tgt}(y|z)))dz\\
\end{aligned}
\end{equation}
\end{proof}

\begin{table*}[tbp]
\centering
\begin{tabular}{c|c|c|c|c|c|c|c|c}
\hline

\multicolumn{1}{c|}{ \multirow{3}*{Methods} }& \multicolumn{4}{c|}{Paper Citation Network (Single-label)} & \multicolumn{4}{c}{Co-author Network (Multi-label)}\\
\cline{2-9}
\multicolumn{1}{c|}{} &\multicolumn{2}{c|}{A$\rightarrow$B}&\multicolumn{2}{c|}{B$\rightarrow$A}&\multicolumn{2}{c|}{A$\rightarrow$B}&\multicolumn{2}{c}{B$\rightarrow$A}\\
\cline{2-9}
\multicolumn{1}{c|}{} &Macro F1&Micro F1&Macro F1&Micro F1&Macro F1&Micro F1&Macro F1&Micro F1\\
\hline
DeepWalk&0.282&0.381&0.22&0.32&0.517&0.646&0.502&0.620\\
\hline
LINE&0.156&0.214&0.175&0.272&0.525&0.634&0.506&0.601\\
\hline
Node2vec&0.147&0.196&0.248&0.32&0.513&0.632&0.520&0.627\\
\hline
GraphSAGE Unsup&0.671&0.703&\textbf{0.861}&0.853&0.724&0.809&0.741&0.832\\
\hline
DANE&\textbf{0.797}&\textbf{0.803}&0.852&\textbf{0.872}&\textbf{0.785}&\textbf{0.847}&\textbf{0.776}&\textbf{0.849}\\
\hline
\end{tabular}
\caption{\label{nodeclassification}Micro and macro F1 score of different network embedding methods in unsupervised domain adaptation}
\end{table*}

This theorem ensures that an embedding algorithm can support domain adaptation better when it : (1) makes $p_{src}(z)$ and $p_{tgt}(z)$ closer; (2) makes $\hat{P}_{src}(y|z)$ and $\hat{P}_{tgt}(y|z)$ closer. Obviously, the former objective can be achieved via distribution alignment. Simultaneously, the latter objective is achieved via embedding space alignment. Because DANE applies shared weight GCN, it represents two nodes having similar local network neighborhoods with similar embedding vectors. Meanwhile, under structural equivalence hypothesis, two nodes having similar local network neighborhoods are likely to have same label. Consequently, embedding space alignment can help keep $\hat{P}_{src}(y|z)$ and $\hat{P}_{tgt}(y|z)$ close almost everywhere in $Z$. 

\section{Experiments}

\subsection{Experiment Settings}

\subsubsection{Datasets}

\begin{table}
\centering
\begin{tabular}{c||c|c}
\Xhline{1.5pt}
Network Name & {\bf Nodes } & {\bf Edges}\\
\Xhline{1.5pt}
Paper Citation A  & 2277  & 8245 \\
\hline
Paper Citation B &  3121 & 7519 \\
\hline
Co-author A & 1500 & 10184 \\
\hline
Co-author B & 1500 & 10606 \\
\Xhline{1.5pt}
\end{tabular}
\caption{\label{tab:unsurpervised} Network size of two datasets.}
\end{table}

\begin{figure}
\centering
\includegraphics[width=3.0in]{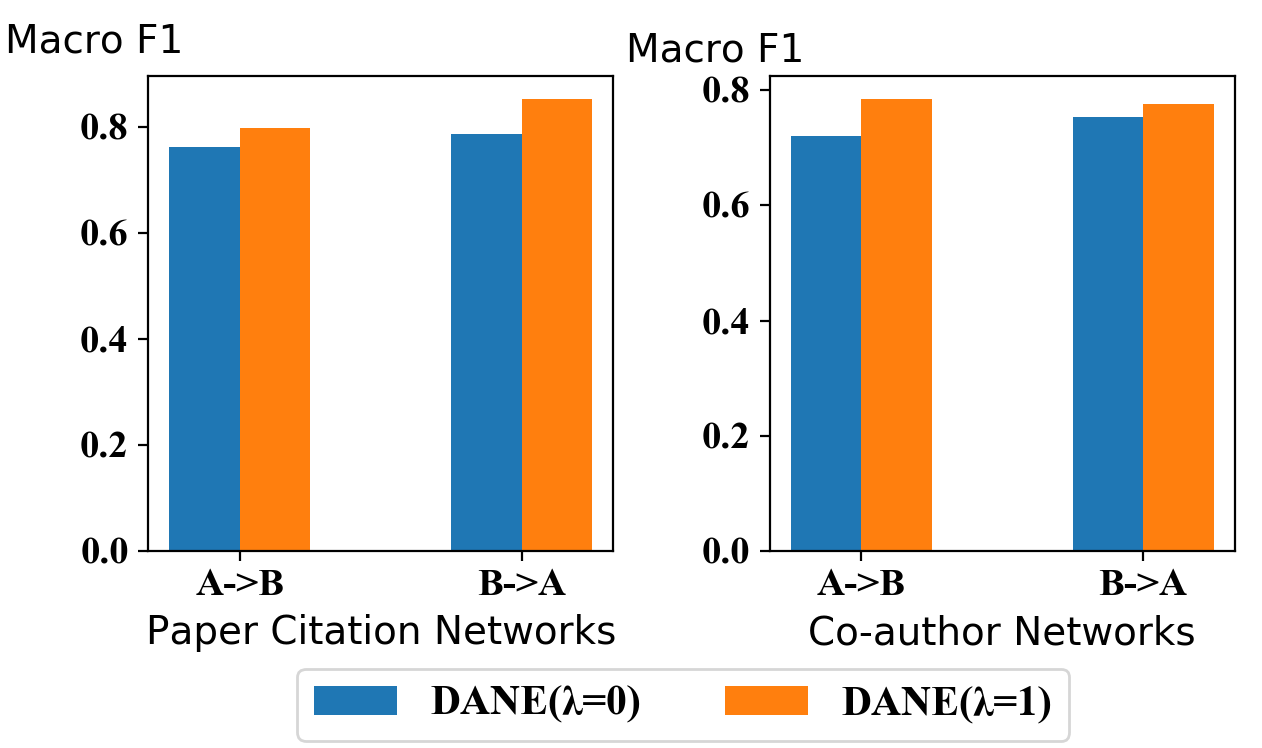}
\caption{Node classification performance of DANE ($\lambda=0$) and DANE ($\lambda=1$) }
\label{fig:2}
\end{figure} 

\begin{figure}
\centering
\includegraphics[width=3.3in]{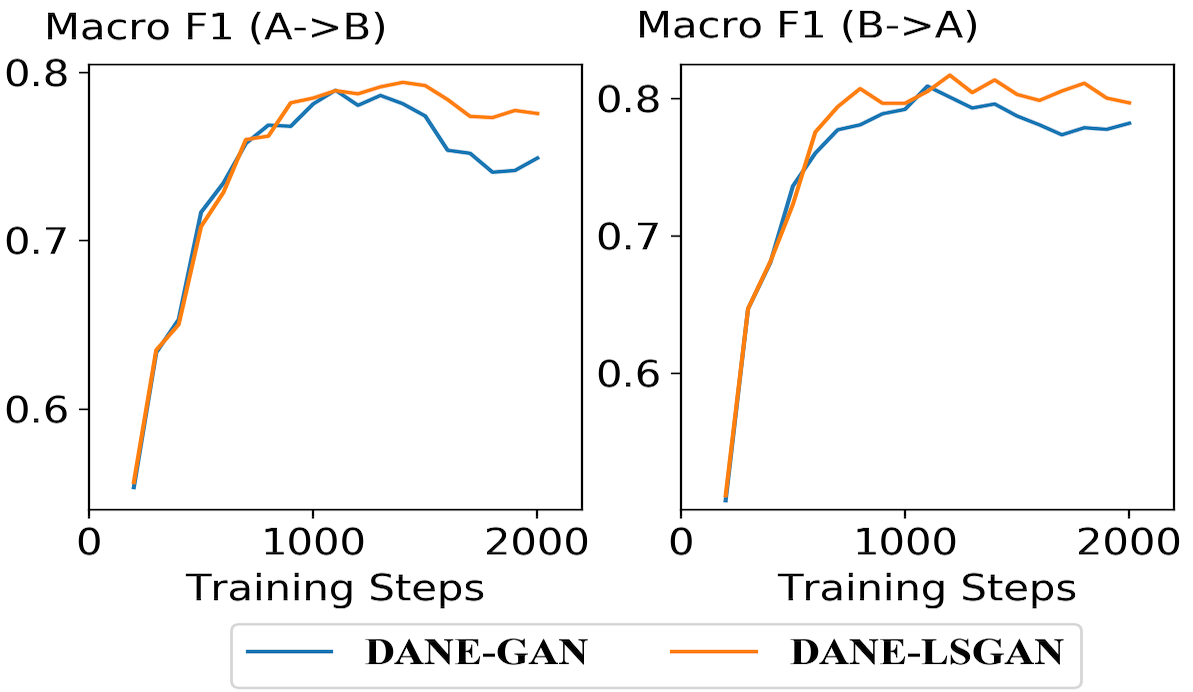}
\caption{Performance vs Training Batches (Paper citation network A $\rightarrow$ B). The orange line performs more steady than the blue line.}
\label{fig:3}
\end{figure} 

\paragraph{Paper Citation Networks.} Paper Citation Networks\footnote{{collected from Aminer database \cite{Tang16Aminer} \label{Aminer}}} consist of two different networks $A$ and $B$, where each node is a paper. The label of each paper is its field. The feature of each node is a word frequency vectors constructed from the abstract of papers.

\paragraph{Co-author Networks.} Co-author Networks\textsuperscript{\ref {Aminer}} consist of two different networks $A$ and $B$, where each node is an author. Each author is assigned with one or more label denoting research topics. The feature of each node is a word frequency vector constructed from the keyword of the author's papers.

\begin{figure*}[tbp]
\centering
\includegraphics[width=6.5in]{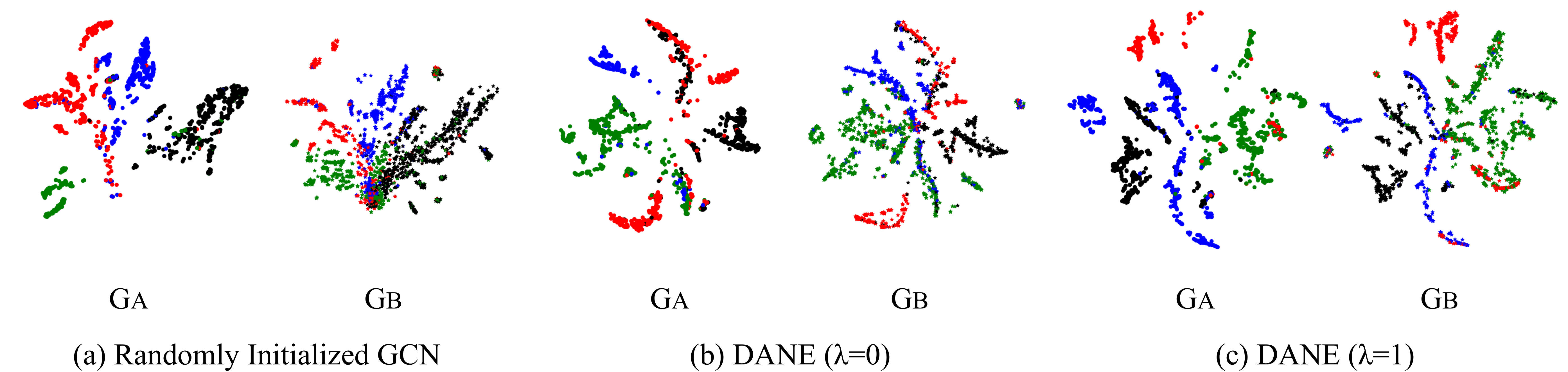}
\caption{Visualization of Paper Citation Networks generated by a randomly initialized GCN, DANE ($\lambda=0$) and DANE ($\lambda=1$).}
\label{fig:4}
\end{figure*}

\subsubsection{Relevant Methods and Evaluation Metrics}
We compare our model with well-recognized network embedding algorithms including DeepWalk, LINE, Node2vec, unsupervised GraphSAGE(GCN version)\cite{Perozzi2014DeepWalk,Tang2015LINE,Grover2016node2vec,Hamilton2017Inductive}. 
When using GraphSAGE, we train it on the source network and directly use the optimized parameter to embed the nodes from target network without further training.
We evaluate these methods by testing their performance on a domain adaptation task of node classification. We train a classifier based on L2-regularized logistic regression via SGD algorithm with the embeddings from source network, then test the performance of the classifier on target network. We adopt micro and macro F1 score to evaluate the performance. 
\subsubsection{Hyperparameter Set-up}
To be fair, for all methods we set the embedding dimension to 128 on \textbf{Paper Citation Networks}, and 32 on \textbf{Co-author Networks}. For methods applying negative sampling, we set negative sampling number as 5. For methods employing GCN, we use same activation function and 2-layer architecture.

\subsection{Domain Adaptation}
In this section, we firstly compare DANE with other baselines by training classifiers on source network embedding and testing their performance on target network embedding. Table \ref{nodeclassification} shows the result of all methods on \emph{Paper Citation Networks} and \emph{Co-author Networks}. DANE outperforms all other methods in knowledge transferring. Besides, we can find following phenomenons:

All deep network embedding algorithms outperforms other methods.
We propose that this is because the local sub-network is suitable for measuring the similarity cross-network node pairs. DeepWalk, LINE and Node2vec all preserve $k$-order proximity, which consider two nodes within $k$ hops are similar. However, the distance between two nodes from different networks is infinity. Consequently, the structural similarity of cross-network node pairs can not be preserved, resulting in poorly aligned embedding spaces and performance close to random guess on target network.

Meanwhile, with same architecture of GCN, DANE outperform GraphSAGE. This phenomenon reflects the contribution of the combination of distribution alignment and the multi-task loss function.

\subsection{Ablation Test}

In this section, by comparing DANE with its variants, we indicate the importance of DANE's unique design.

First, to analyze the importance of distribution alignment, we compare complete DANE, denoted as DANE ($\lambda=1$), and its variant which has no adversarial learning regularization, denoted as DANE ($\lambda=0$). The result is shown in Fig \ref{fig:2}. DANE ($\lambda=1$) performs better in both transferring directions on two networks comparing to its own variant that has no adversarial learning regularization. This phenomenon indicates the importance of adversarial learning regularization. 

Second, we replace the LSGAN-based adversarial learning regularization with GAN-based regularization while preserving all other hyper-parameter settings. Fig \ref{fig:3} is a line chart reflecting the relation between number of training batches and model's performance on target network. Compared with the LSGAN-based DANE, the GAN-based DANE suffers from a more serious performance drop after reaching the peak\footnote{{Although the GAN-variant performs as steadily as the LSGAN-based version under some random seeds, it is difficult to find those 'steady' seeds when labels on target network are scarce.}}.

\subsection{Embedding Visualization}

In this section, to understand the advantage of distribution alignment better, we visualize the embedding of two paper citation networks generated by DANE ($\lambda=1$), DANE ($\lambda=0$) and a randomly initialized GCN. The randomly initialized GCN also embeds two networks via shared parameters. We use t-SNE \cite{Maaten2008Visualizing} package to reduce the dimensionality of embedding vectors to $2$. The result is shown in Fig \ref{fig:4}. The color of each node represent its label.

Due to shared parameter architecture, all three algorithms achieves that most nodes having similar representations have same labels. This phenomenon indicates the advantage of share parameter architecture: it can align embedding space without requirement for training. After further training without distribution alignment, DANE ($\lambda=0$) succeeds in making the distribution of two network's embedding similar. Furthermore, when adding distribution alignment, DANE ($\lambda=1$) can generate embeddings with more similar distributions in the area nearby boundary line of different labels than DANE ($\lambda=0$). Therefore, DANE will have better performance in model transferring task.

\section{Conclusion and Future Works}

In this paper, we formulate the task of unsupervised network embedding supporting domain adaptation on multiple domain compatible networks. To the best of our knowledge, we are the first team to propose this significant task. Meanwhile, we propose Domain Adaptive Network Embedding (DANE) to handle this task. Specifically, we apply a shared weight graph convolutional network architecture with constraints of adversarial learning regularization. Empirically, we verify DANE's performance in a variety of domain compatible network
datasets. The extensive experimental results of node classification and embedding visualization indicate the advantages of DANE. For future work, one intriguing direction is to generalize DANE to solve the domain adaptation on heterogeneous networks. Also, improving the framework of DANE to solve semi-supervised and supervised domain adaptation will be beneficial for diverse scenarioes.

\section*{Acknowledgments}

This work was supported by the National Natural Science Foundation of China (Grant No. 61876006 and No. 61572041).
\appendix







\bibliographystyle{named}
\bibliography{ijcai19}

\end{document}